\documentclass[twoside,11pt]{article}

\usepackage[preprint]{jmlr2e}
\usepackage{amsmath,amssymb}
\usepackage{booktabs}
\usepackage{algorithm}
\usepackage{algpseudocode}
\usepackage{enumitem}
\usepackage{lastpage}
\usepackage{graphicx}
\usepackage{array}
\usepackage{multirow}
\usepackage{url}
\usepackage{tabularx}

\numberwithin{equation}{section}

\newcommand{\softmax}{\operatorname{softmax}}
\newcommand{\KL}{\operatorname{KL}}
\newcommand{\CE}{\operatorname{CE}}
\newcommand{\E}{\mathbb{E}}

\newcommand{\ddorm}{\textbf{DDO-RM}}

\newcommand{\takeaway}[2]{%
  \par\smallskip
  \noindent\textbf{#1.} #2\par\smallskip
}

\newtheorem{proposition}{Proposition}
\newtheorem{remark}{Remark}

\jmlrheading{0}{2026}{1--\pageref{LastPage}}{}{}{00-0000}{Tiantian Zhang and Jierui Zuo}
\ShortHeadings{DDO-RM for LLM Preference Optimization}{Zhang and Zuo}
\firstpageno{1}

\title{DDO-RM: Distribution-Level Policy Improvement after Reward Learning}

\author{\name Tiantian Zhang \email t.zhang8@columbia.edu \\
       \addr Department of Computer Science \\
       Columbia University, New York, NY 10027, USA
       \AND
       \name Jierui Zuo \email zuojr22@mails.tsinghua.edu.cn \\
       \addr Department of Management Science and Engineering \\
       Tsinghua University, Beijing, China
       \AND
       \name Michael Chen\email yc4131@columbia.edu \\
       \addr Department of Computer Science \\
       Columbia University, New York, NY 10027, USA
       \AND
       \name Wenping Wang \email wenpingw@alumni.cmu.edu \\
       \addr Department of Computer Science \\
       Carnegie Mellon University, Pittsburgh, PA 15213, USA}
\editor{}

\begin{document}

\maketitle

\begin{abstract}
Recent theory on the RLHF--DPO dichotomy shows that reward-model-first methods and direct policy-learning methods can differ fundamentally under representation and finite-sample constraints. In particular, when the reward function is statistically simpler than the induced autoregressive policy, explicitly learning a reward model can be more sample-efficient than directly fitting the policy. Motivated by this view, we study the second stage of reward-model-based preference optimization: how should a learned reward model be converted into a policy update?

We propose \textbf{DDO-RM}, a finite-candidate decision-optimization method that converts reward-model scores into an explicit target distribution over candidate responses. Instead of optimizing only a pairwise preference objective or applying stochastic policy-gradient updates, DDO-RM performs a KL-regularized mirror-descent update on the candidate decision distribution. For each prompt, the method centers reward scores under the current policy, constructs a Boltzmann-style target distribution, and distills this target back into the language model.

This reframes preference optimization as \emph{distribution-level decision optimization}: DPO directly fits the policy from preference pairs, PPO-based RLHF optimizes sampled reward gradients after reward learning, while DDO-RM explicitly projects the policy toward the reward-improved distribution over a finite candidate set. As a preliminary empirical reference, we evaluate DDO-RM on \texttt{EleutherAI/pythia-410m} with \texttt{HuggingFaceH4/ultrafeedback\_binarized} and compare against DPO on the held-out \texttt{test\_prefs} split. DDO-RM improves mean pair accuracy from 0.5238 to 0.5602, AUC from 0.5315 to 0.5382, and mean margin from 0.1377 to 0.5353 across three seeds. These results are encouraging but still preliminary: the study covers one model family, one dataset, one held-out split, and three seeds. The main contribution is a framework connecting reward-model learning, finite-candidate decision distributions, and mirror-descent policy improvement.
\end{abstract}


\section{Introduction}

Preference optimization has become a standard stage in large language model alignment. A common setup presents a prompt together with a chosen and a rejected response, then trains a policy so that the preferred response receives a higher score. Direct Preference Optimization (DPO) is influential because it turns preference learning into a stable direct policy objective without running an explicit online reinforcement-learning loop.

This paper takes a different starting point. We view preference optimization as having two logically distinct steps:
\begin{enumerate}[leftmargin=*]
    \item infer reward information from preference data;
    \item use that reward information to improve the policy.
\end{enumerate}
DPO combines these two steps into direct policy fitting. Classical reward-model-based RLHF separates them: first learn a reward model, then optimize the policy using PPO-style reinforcement learning. Recent work on the RLHF--DPO performance gap shows that neither route is universally superior: RLHF, DPO, or online DPO can outperform one another depending on representation and finite-sample conditions. Crucially for our purposes, that work also identifies regimes where reward learning can be statistically more efficient than direct policy learning, especially when the reward function is simpler or more sparse than the induced autoregressive policy.

Motivated by this reward-model-first perspective, we ask a narrower question: once a reward model is available, can the policy-improvement stage be made more direct and distribution-aware than PPO? We answer this question through \ddorm{}, a finite-candidate decision-optimization method. For each prompt, DDO-RM takes a candidate response set, evaluates candidates with a reward model, constructs a reward-improved target distribution over the candidates, and distills that distribution into the policy.

\takeaway{Positioning}{We build on recent theory separating reward-model-first and direct-policy learning. Our contribution is not to re-prove that dichotomy, but to study the second stage of the reward-model-first pipeline: \emph{how should a learned reward model be converted into a policy update?} We propose DDO-RM, a finite-candidate mirror-descent policy-improvement operator that explicitly constructs a reward-improved decision distribution and distills it into the model.}

The core distinction is not merely algorithmic. DPO optimizes a pairwise comparison landscape. PPO-RLHF optimizes expected reward through sampled stochastic gradients. DDO-RM instead optimizes an explicit \emph{decision distribution} over the candidate set. In this sense, DDO-RM is best understood as a distribution-level policy-improvement operator for the reward-model-first alignment pipeline, rather than simply as another pairwise preference loss.

\paragraph{Contributions.}
This manuscript makes four focused contributions.
\begin{enumerate}[leftmargin=*]
    \item \textbf{Formulation.} We formulate reward-model-based preference optimization as finite-candidate decision optimization over a prompt-conditioned candidate set.
    \item \textbf{Method.} We propose DDO-RM, which turns learned reward scores into a KL-regularized reward-improved target distribution and distills this target into the policy.
    \item \textbf{Geometric clarification.} We identify DDO-RM as an instantiation of classical KL mirror descent on a finite candidate simplex, clarifying how reward-model scores induce a conservative policy-improvement target.
    \item \textbf{Preliminary benchmark.} We provide a reproducible held-out benchmark against DPO on Pythia-410M / UltraFeedback Binarized, while explicitly identifying PPO-RLHF and GRPO as the more direct next comparators for reward-model-based settings.
\end{enumerate}

The code, configurations, outputs, and plotting scripts for the current benchmark are available at:
\begin{center}
\url{https://github.com/zuojr/ddorm-llm-preference-benchmark}
\end{center}

\paragraph{Scope of this paper.}
The manuscript is intentionally narrow. It does not claim that reward-model-first learning is always superior to direct preference optimization. It also does not claim that the current DPO comparison is the final apples-to-apples baseline for a reward-model method. The most direct future comparator for DDO-RM is PPO-based RLHF using the same reward model. DPO remains important, but it is a direct-policy baseline rather than a same-information reward-model baseline.

\takeaway{Introduction takeaway}{DDO-RM is not proposed as another pairwise preference loss. It is a policy-improvement operator for the reward-model-first pipeline.}

\section{Background: RLHF, DPO, and the Reward-Model Dichotomy}

\subsection{RLHF as reward learning plus policy improvement}

In standard reward-model-based RLHF, one first trains a scalar reward model $\hat r(x,y)$ from human preference comparisons. The policy is then improved by optimizing a KL-regularized objective of the form
\begin{equation}
    \max_{\pi} \; \E_{x,y\sim \pi(\cdot\mid x)}[\hat r(x,y)] - \lambda \, \KL\bigl(\pi(\cdot\mid x)\,\|\,\pi_{\rm ref}(\cdot\mid x)\bigr).
\label{eq:rlhf}
\end{equation}
In practice, this second step is often optimized by PPO or a related policy-gradient method. The reward model is noisy and only learned from comparisons, but it provides a scalar signal that can be reused across generated candidates.

\subsection{DPO as direct policy fitting}

DPO avoids a separately optimized RL loop. For a chosen response $y^+$ and rejected response $y^-$, the standard DPO loss is
\begin{equation}
\mathcal{L}_{\mathrm{DPO}}(x,y^+,y^-)
=
-\log \sigma\!\Bigl(
\beta \bigl[
\log \pi_\theta(y^+\mid x)-\log \pi_\theta(y^-\mid x)
-\log \pi_{\mathrm{ref}}(y^+\mid x)+\log \pi_{\mathrm{ref}}(y^-\mid x)
\bigr]
\Bigr),
\label{eq:dpo}
\end{equation}
where $\pi_{\mathrm{ref}}$ is the reference policy. DPO is a strong baseline because it uses the pairwise supervision signal directly and avoids the variance and engineering complexity of online RL.

\subsection{Why reward-model-first can still matter}

The RLHF--DPO dichotomy suggests that the choice between direct policy learning and reward-model-first learning depends on representation and finite-sample structure. If the reward function is easier to represent or estimate than the induced optimal autoregressive policy, then learning the reward first can be statistically advantageous. However, if the reward model class is mis-specified, direct policy learning can be preferable. This paper builds on the reward-model-first side of that dichotomy but does not claim that reward modeling is always superior.

\paragraph{How we use the dichotomy theory.}
We use the RLHF--DPO dichotomy as motivation rather than as a theorem that directly proves DDO-RM superiority. The strongest separations in that theory are based on stylized constructions, and their generality depends on assumptions about the reward class, policy class, and parameterization. Our paper only relies on the weaker and more robust lesson: reward-model-first learning is a meaningful regime, and in some settings the learned reward can be a statistically simpler object than the induced autoregressive policy. DDO-RM then studies a different question: once such a reward model is available, how should it be converted into a policy update?

Our contribution is therefore focused on the \emph{second} stage: given a learned reward model, DDO-RM replaces noisy sampled policy-gradient improvement with an explicit finite-candidate distribution update.

\takeaway{Background takeaway}{The dichotomy motivates reward-model-first learning in some regimes, but it does not prescribe the best second-stage policy update. DDO-RM addresses that second-stage question.}

\section{Related Work and Novelty Clarification}
\label{sec:related}

\subsection{Reward-model-first RLHF and PPO}
Classical RLHF separates preference learning into reward modeling and policy optimization. A reward model is trained from comparison data and is then used to optimize a KL-regularized policy objective, commonly with PPO or related policy-gradient methods \cite{ouyang2022training,ppo}. This two-stage route is computationally more involved than direct preference objectives, but it has an important conceptual advantage: the reward model is a reusable scalar evaluator of generated candidates. Recent theory on the RLHF--DPO gap further shows that this separation can be statistically advantageous when the reward function is easier to estimate than the induced optimal autoregressive policy \cite{rlhfdpo}. DDO-RM builds on this reward-model-first route, but focuses on the second stage: how to convert the learned reward into a policy update.

\paragraph{Scope of the borrowed theory.}
The RLHF--DPO dichotomy is used here as a positioning result, not as a black-box proof of our method. Its exact separation results are deliberately stylized and depend on particular reward and policy classes. We therefore avoid claiming that the dichotomy universally favors reward-model-first learning. Instead, it motivates the setting in which DDO-RM is most relevant: when a reward model is learned or available, and the remaining algorithmic question is whether policy improvement should be done by sampled PPO-style gradients or by an explicit finite-candidate distribution update.

\takeaway{Theory-use takeaway}{We borrow the RLHF--DPO dichotomy only as motivation for reward-model-first alignment. DDO-RM's own contribution is the second-stage policy-improvement operator, not a new proof of the dichotomy or a universal claim that reward models always dominate direct policy learning.}

\subsection{Direct preference optimization and pairwise variants}
DPO eliminates explicit reward-model training and online RL by directly optimizing the policy from chosen--rejected pairs \cite{dpo}. This has made DPO a strong and widely used direct-preference baseline. Several later methods modify the direct-preference objective, for example by using prospect-theoretic utilities in KTO \cite{kto}, odds-ratio penalties in ORPO \cite{orpo}, or reference-free average-log-probability rewards in SimPO \cite{simpo}. These methods are closest to DDO-RM in the broad goal of preference alignment, but they differ in information structure: they primarily update the policy from pairwise or binary preference signals, whereas DDO-RM assumes an explicit reward model and constructs a full target distribution over a finite candidate set.

\subsection{Group-relative and verifiable-reward optimization}
GRPO and related group-relative methods are most natural when multiple completions are generated for the same prompt and scored by a verifier or reward function \cite{grpo}. This makes GRPO a closer comparator to DDO-RM than DPO in multi-candidate or reasoning settings. The key difference is algorithmic geometry. GRPO uses group-relative reward normalization inside a policy-gradient style update, while DDO-RM explicitly constructs the reward-improved finite-candidate distribution and distills that distribution into the policy.

\subsection{Mirror descent, exponentiated gradients, and KL policy improvement}
The KL update used in DDO-RM is not new as an optimization primitive. Entropy mirror descent, exponentiated-gradient updates, and KL-regularized policy-improvement steps are classical tools in online learning, convex optimization, and reinforcement learning \cite{nemirovsky1983problem,beck2003mirror,kivinen1997exponentiated}. In the finite-candidate simplex, the solution of a linearized reward objective plus a KL penalty has the Boltzmann form
\[
q_i \propto p_i \exp(\eta r_i).
\]
The novelty of DDO-RM is therefore not the proximal operator itself. The novelty is its role as a finite-candidate policy-improvement layer for reward-model-first LLM alignment: learned candidate-level rewards are converted into an explicit decision distribution, and the language model is trained to match that distribution.

\subsection{What is new and what is inherited}
Table~\ref{tab:novelty} summarizes the intended novelty boundary. We make this boundary explicit because otherwise the method can be mistaken for a claim that KL mirror descent or Boltzmann policy improvement is new.

\begin{table}[H]
\centering
\caption{Novelty boundary. DDO-RM uses a classical KL mirror-descent operator, but applies it as the second-stage policy-improvement map in reward-model-first LLM alignment.}
\label{tab:novelty}
\small
\begin{tabularx}{\linewidth}{>{\raggedright\arraybackslash}p{0.34\linewidth}X}
\toprule
Component & Status in this paper \\
\midrule
Reward modeling from human preferences & inherited from the RLHF pipeline \\
DPO-style direct pairwise policy fitting & baseline / related work \\
KL mirror descent and exponentiated-gradient update on the simplex & classical optimization tool, not claimed as new \\
Finite-candidate decision distribution for LLM post-training & proposed formulation \\
Using a learned reward model to construct a reward-improved target distribution over candidates & proposed DDO-RM policy-improvement step \\
Distilling the reward-improved distribution into the language model & proposed training interface \\
Same-RM comparison against PPO-RLHF and GRPO & future fair empirical test beyond the current minimal benchmark \\
\bottomrule
\end{tabularx}
\end{table}

\takeaway{Related-work takeaway}{DDO-RM does not claim that KL mirror descent is new. Its contribution is to use this classical operator as an explicit finite-candidate policy-improvement layer after reward learning, in contrast to pairwise direct objectives and sampled PPO-style updates.}

\section{Finite-Candidate Decision Setup}

Let $x$ denote a prompt and let $\mathcal{Y}(x)=\{y_1,\dots,y_K\}$ denote a finite candidate set. A policy model assigns a scalar score
\[
 s_i = s_\theta(x,y_i), \qquad i=1,\dots,K,
\]
which in the current implementation is the average token log-probability of the candidate response under the policy. These scores induce a temperature-controlled decision distribution
\begin{equation}
 p_i = \softmax(s/\tau)_i = \frac{\exp(s_i/\tau)}{\sum_{j=1}^K \exp(s_j/\tau)}.
\label{eq:policydist}
\end{equation}

In the public benchmark used in this paper, each prompt has only two candidates,
\[
\mathcal{Y}(x)=\{y^+,y^-\},
\]
where $y^+$ is the chosen response and $y^-$ is the rejected response. The DDO-RM formulation, however, is written for arbitrary finite $K$, which makes it more natural for reranking, recommendation, listwise preference learning, and candidate-based reasoning than a purely pairwise objective.

\section{DDO-RM: Distribution-Level Policy Improvement}
\label{sec:ddorm}

DDO-RM starts from the current policy distribution $p$ over the candidate set and introduces a reward-model score
\[
 r_i = \hat r(x,y_i)
\]
for each candidate. Define the policy-weighted average reward and centered reward by
\begin{equation}
\bar r = \sum_{j=1}^K p_j r_j,
\qquad
\tilde r_i = r_i - \bar r.
\label{eq:centeredreward}
\end{equation}
The centered form is shift-invariant and subtracts the reward baseline under the current policy. DDO-RM then forms an updated score vector
\begin{equation}
 s'_i = s_i + \eta \, \tilde r_i,
\label{eq:ddo-step}
\end{equation}
where $\eta>0$ is the decision step size. The reward-guided target distribution is
\begin{equation}
 q_i = \softmax(s'/\tau)_i
= \frac{\exp\bigl((s_i+\eta\tilde r_i)/\tau\bigr)}{\sum_{j=1}^K \exp\bigl((s_j+\eta\tilde r_j)/\tau\bigr)}.
\label{eq:qdist}
\end{equation}
Finally, the policy is trained to match this target distribution with a cross-entropy objective,
\begin{equation}
\mathcal{L}_{\mathrm{DDO\mbox{-}RM}}(x,\mathcal{Y}(x))
= \CE(q,p_\theta)
= -\sum_{i=1}^K q_i \log p_i.
\label{eq:ddoloss}
\end{equation}

\takeaway{Method takeaway}{DDO-RM turns reward scores into a full decision distribution. The update is not merely ``choose the highest-reward response''; it is a conservative KL projection from the current policy toward a reward-improved finite-candidate distribution.}

\begin{proposition}[KL policy-improvement interpretation]
For fixed current distribution $p\in\Delta_K$ and reward vector $r\in\mathbb{R}^K$, the DDO-RM target distribution in Equation~\eqref{eq:qdist} is equivalent to
\begin{equation}
q
= \arg\max_{u\in\Delta_K}
\left\{
\langle u,r\rangle - \frac{\tau}{\eta}\KL(u\|p)
\right\}.
\label{eq:kl-proj}
\end{equation}
\end{proposition}

\begin{proof}
The first-order optimality condition for Equation~\eqref{eq:kl-proj} gives
\[
 r_i - \frac{\tau}{\eta}\log\frac{u_i}{p_i} - \alpha = 0,
\]
where $\alpha$ is the simplex normalization multiplier. Thus
\[
 u_i \propto p_i\exp(\eta r_i/\tau).
\]
Since $p_i\propto \exp(s_i/\tau)$, this is equivalent to
\[
 u_i \propto \exp((s_i+\eta r_i)/\tau).
\]
Replacing $r_i$ by the centered reward $\tilde r_i=r_i-\bar r$ only multiplies all unnormalized probabilities by the same constant and therefore leaves the normalized distribution unchanged. Hence $u=q$.
\end{proof}

This proposition identifies DDO-RM as a KL-regularized policy-improvement operator on the candidate simplex. It should be read as a geometric clarification rather than as a claim that the KL update itself is new.

\begin{remark}[Classical operator, new alignment role]
The KL-regularized update in Proposition~1 is a classical entropy mirror-descent or exponentiated-gradient step on the simplex. We do not claim novelty in the proximal operator itself. The contribution of DDO-RM is its role in the reward-model-first alignment pipeline: after reward learning, DDO-RM uses this operator to convert learned candidate-level rewards into an explicit reward-improved decision distribution, and then distills that distribution into the policy. Thus the novelty is not the KL update alone, but the finite-candidate policy-improvement formulation for LLM preference optimization.
\end{remark}

\subsection{Algorithm and pipeline}

\begin{algorithm}[H]
\caption{DDO-RM training on a finite candidate set}
\label{alg:ddorm}
\begin{algorithmic}[1]
\Require prompt $x$, candidate set $\mathcal{Y}(x)=\{y_1,\dots,y_K\}$, policy model $s_\theta$, reward model $\hat r$, temperature $\tau$, step size $\eta$
\State Compute policy scores $s_i \leftarrow s_\theta(x,y_i)$ for all candidates.
\State Form policy distribution $p \leftarrow \softmax(s/\tau)$.
\State Compute reward scores $r_i \leftarrow \hat r(x,y_i)$.
\State Center rewards: $\tilde r_i \leftarrow r_i - \sum_j p_j r_j$.
\State Form target distribution $q \leftarrow \softmax((s + \eta \tilde r)/\tau)$.
\State Update $\theta$ by minimizing $\CE(q,p_\theta)$.
\end{algorithmic}
\end{algorithm}

\begin{figure}[H]
\centering
\fbox{\begin{minipage}{0.93\linewidth}
\centering
\begin{tabular}{@{} l @{\quad} p{0.72\linewidth} @{}}
\textbf{DPO} & preference pair $\to$ pairwise logistic policy update \\[4pt]
\textbf{PPO-RLHF} & preference data $\to$ reward model $\to$ sampled rollouts $\to$ stochastic policy-gradient update \\[4pt]
\textbf{DDO-RM} & preference data $\to$ reward model $\to$ candidate rewards $\to$ explicit KL / mirror-descent target $q$ $\to$ distillation \\
\end{tabular}
\end{minipage}}
\caption{Pipeline comparison. DDO-RM shares the reward-model-first information structure of PPO-RLHF, but replaces sampled policy-gradient improvement with an explicit finite-candidate distribution update.}
\label{fig:pipeline}
\end{figure}

\section{Relation to DPO, PPO-RLHF, and GRPO}

The most important repositioning is that DDO-RM should not be interpreted only as a DPO competitor. It is a reward-model-first policy-improvement method. Table~\ref{tab:method-compare} summarizes the relevant baselines.

\begin{table}[t]
\centering
\caption{Comparison of preference-optimization methods. DDO-RM is closest to PPO-RLHF in information structure because both use an explicit reward model. DPO is a direct-preference reference baseline.}
\label{tab:method-compare}
\small
\begin{tabularx}{\linewidth}{l c >{\raggedright\arraybackslash}p{0.18\linewidth} >{\raggedright\arraybackslash}p{0.28\linewidth} X}
\toprule
Method & RM? & Candidate structure & Update type & Comparison role \\
\midrule
DPO & No & Pair & Direct pairwise policy fitting & Direct-preference baseline \\
PPO-RLHF & Yes & Sampled responses & Stochastic policy-gradient update & Main reward-model baseline \\
GRPO & Yes / verifier & Group completions & Group-relative reward update & Multi-candidate or reasoning baseline \\
DDO-RM & Yes & Finite candidate set & Explicit KL / mirror-descent projection & Proposed reward-model policy-improvement method \\
\bottomrule
\end{tabularx}
\normalsize
\end{table}

This distinction matters for experimental design. The current benchmark against DPO is useful because DPO is the standard direct-preference reference. However, it is not the final same-information comparison: DPO does not train or use an explicit reward model, while DDO-RM does. The clean apples-to-apples reward-model comparison is DDO-RM versus PPO-based RLHF using the same reward model, same prompt distribution, and matched generation budget. In multi-completion or reasoning settings, GRPO is also a natural comparator.

\takeaway{Comparison takeaway}{DPO is a useful direct-preference reference, but PPO-RLHF is the fairest reward-model baseline for DDO-RM. GRPO becomes the natural comparator in multi-completion or verifiable-reward reasoning settings.}

\section{Experimental Setup}
\label{sec:setup}

\subsection{Benchmark configuration}

The current public benchmark uses the following configuration.

\begin{table}[t]
\centering
\caption{Minimal held-out benchmark configuration.}
\label{tab:setup}
\begin{tabular}{>{\raggedright\arraybackslash}p{0.33\linewidth}p{0.58\linewidth}}
\toprule
Item & Value \\
\midrule
Base model & \texttt{EleutherAI/pythia-410m} \\
Dataset & \texttt{HuggingFaceH4/ultrafeedback\_binarized} \\
Training splits used by scripts & \texttt{train\_sft}, \texttt{train\_prefs} \\
Evaluation split & \texttt{test\_prefs} \\
Seeds & 42, 13, 3407 \\
Compared methods & DPO, DDO-RM \\
Metrics & Pair accuracy, AUC, mean margin \\
Code and outputs & \url{https://github.com/zuojr/ddorm-llm-preference-benchmark} \\
\bottomrule
\end{tabular}
\end{table}

This is a held-out evaluation protocol rather than a broad ablation study. The repository keeps the training and plotting pipeline lightweight and focuses on reproducibility: code, configurations, logs, metrics JSONs, figures, and summary tables are public, while heavy checkpoints are intentionally excluded.

\subsection{Evaluation metrics}

For each example $j$ in the held-out pairwise split, let $c_j$ be the policy score of the chosen response and let $r_j$ be the policy score of the rejected response. Define the margin
\begin{equation}
 m_j = c_j - r_j.
\end{equation}
The three reported metrics are:
\begin{align}
\text{Pair Accuracy} &= \frac{1}{N}\sum_{j=1}^N \mathbf{1}[m_j > 0], \\
\text{AUC} &= \mathrm{ROC\mbox{-}AUC}\bigl([1,\dots,1,0,\dots,0],[c_1,\dots,c_N,r_1,\dots,r_N]\bigr), \\
\text{Mean Margin} &= \frac{1}{N}\sum_{j=1}^N m_j.
\end{align}
Pair accuracy measures strict pairwise correctness, AUC measures ranking quality over the concatenated chosen and rejected scores, and mean margin measures the average separation strength between preferred and dispreferred responses.

\section{Preliminary Results}
\label{sec:results}
\begin{figure}[t]
\centering
\begin{minipage}[t]{0.32\linewidth}
  \centering
  \includegraphics[width=\linewidth]{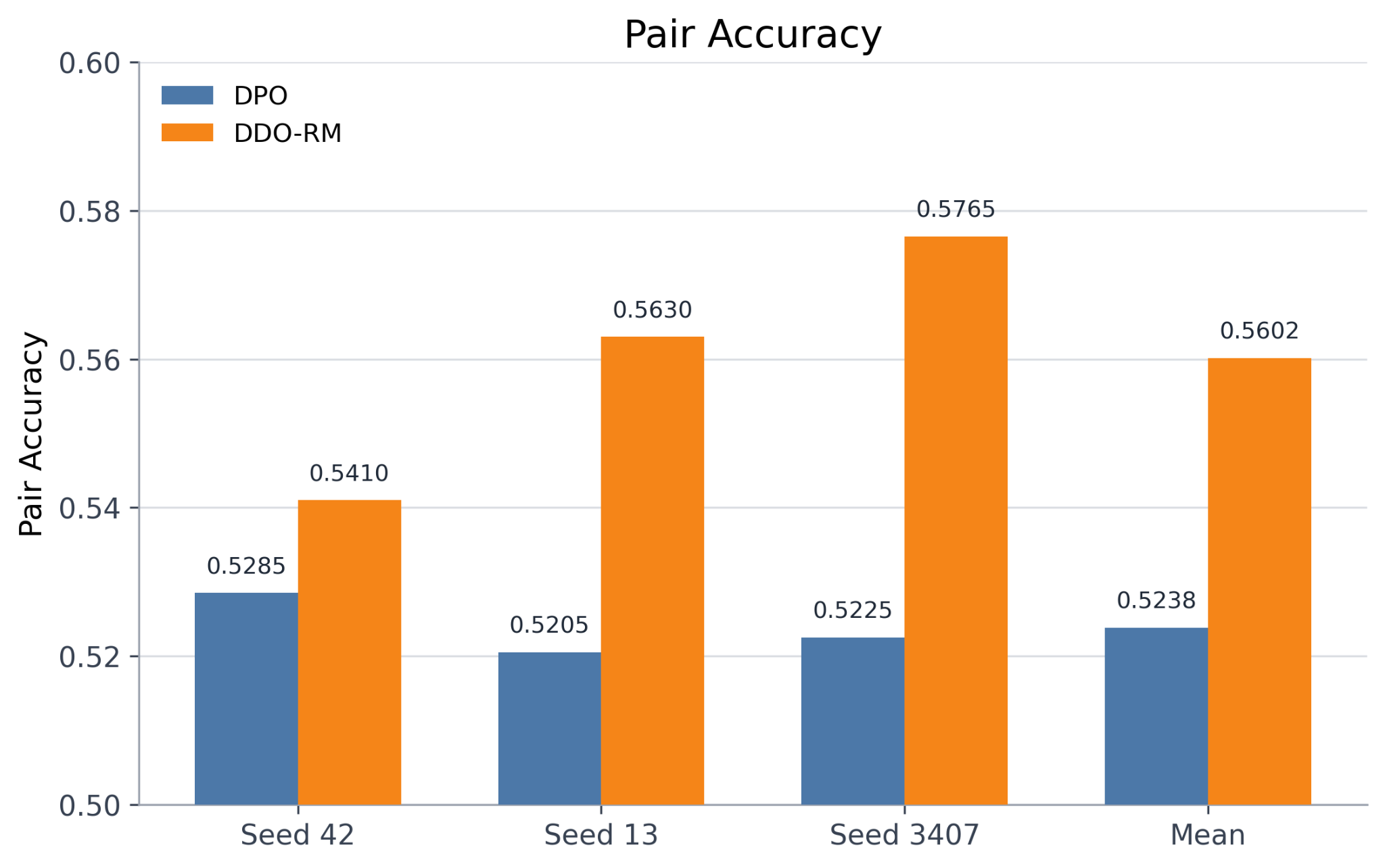}
\end{minipage}
\hfill
\begin{minipage}[t]{0.32\linewidth}
  \centering
  \includegraphics[width=\linewidth]{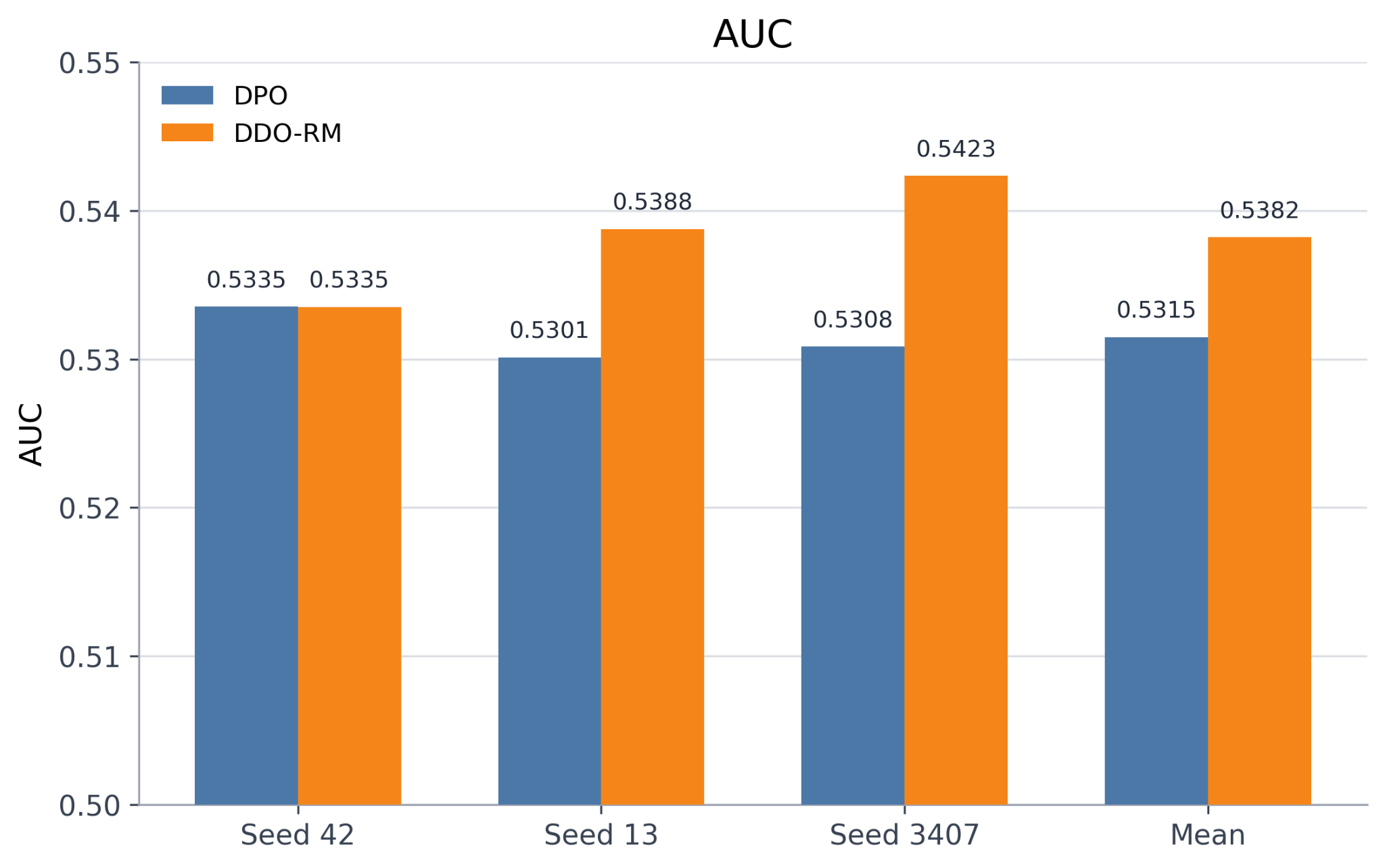}
\end{minipage}
\hfill
\begin{minipage}[t]{0.32\linewidth}
  \centering
  \includegraphics[width=\linewidth]{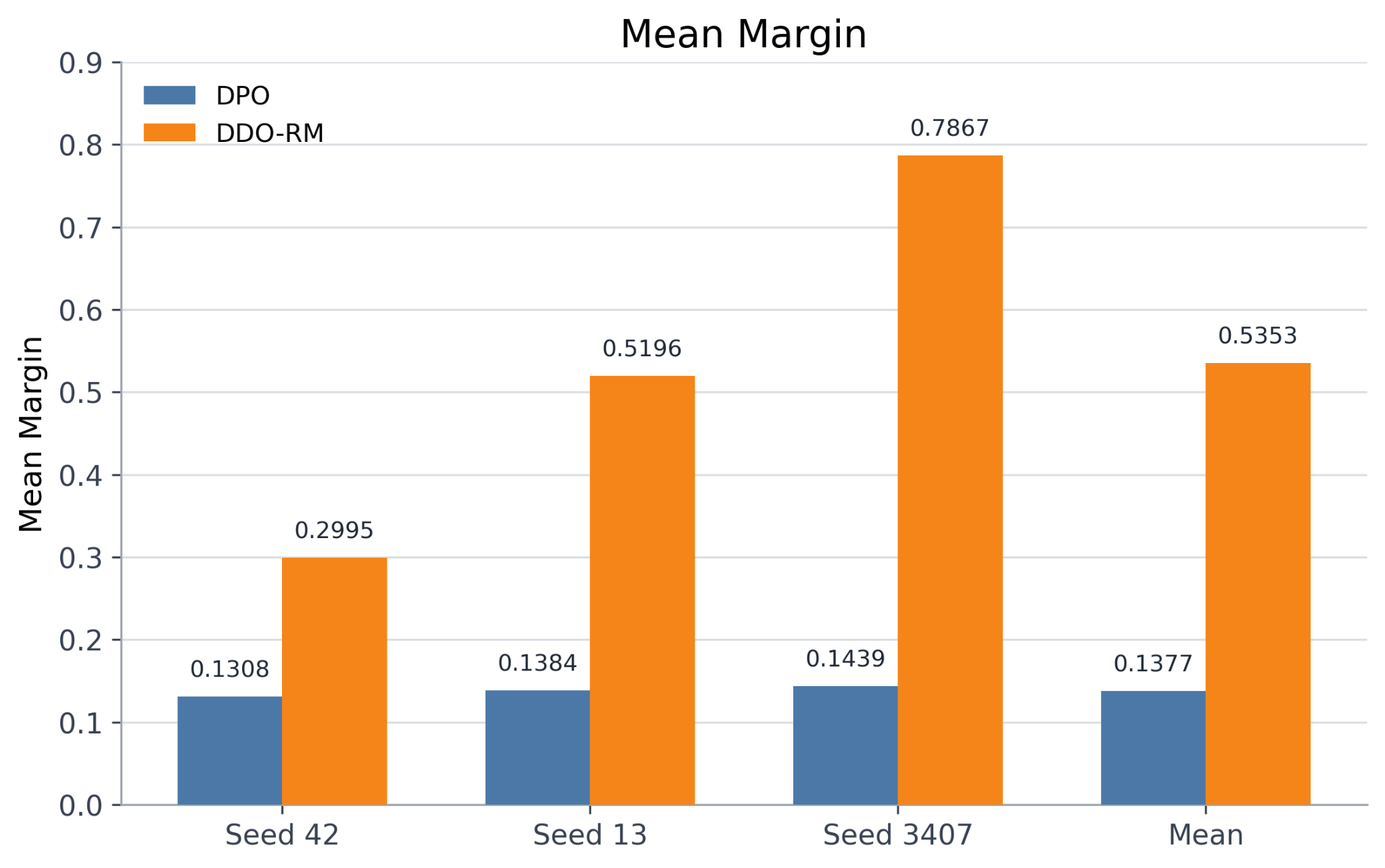}
\end{minipage}
\caption{Mean metric comparison between DPO and DDO-RM across three seeds. From left to right: pair accuracy, AUC, and mean margin. DDO-RM outperforms DPO on all three metrics, with the largest absolute gain on mean margin.}
\label{fig:bar_comparison}
\end{figure}

\begin{figure}[t]
\centering
\includegraphics[width=0.65\linewidth]{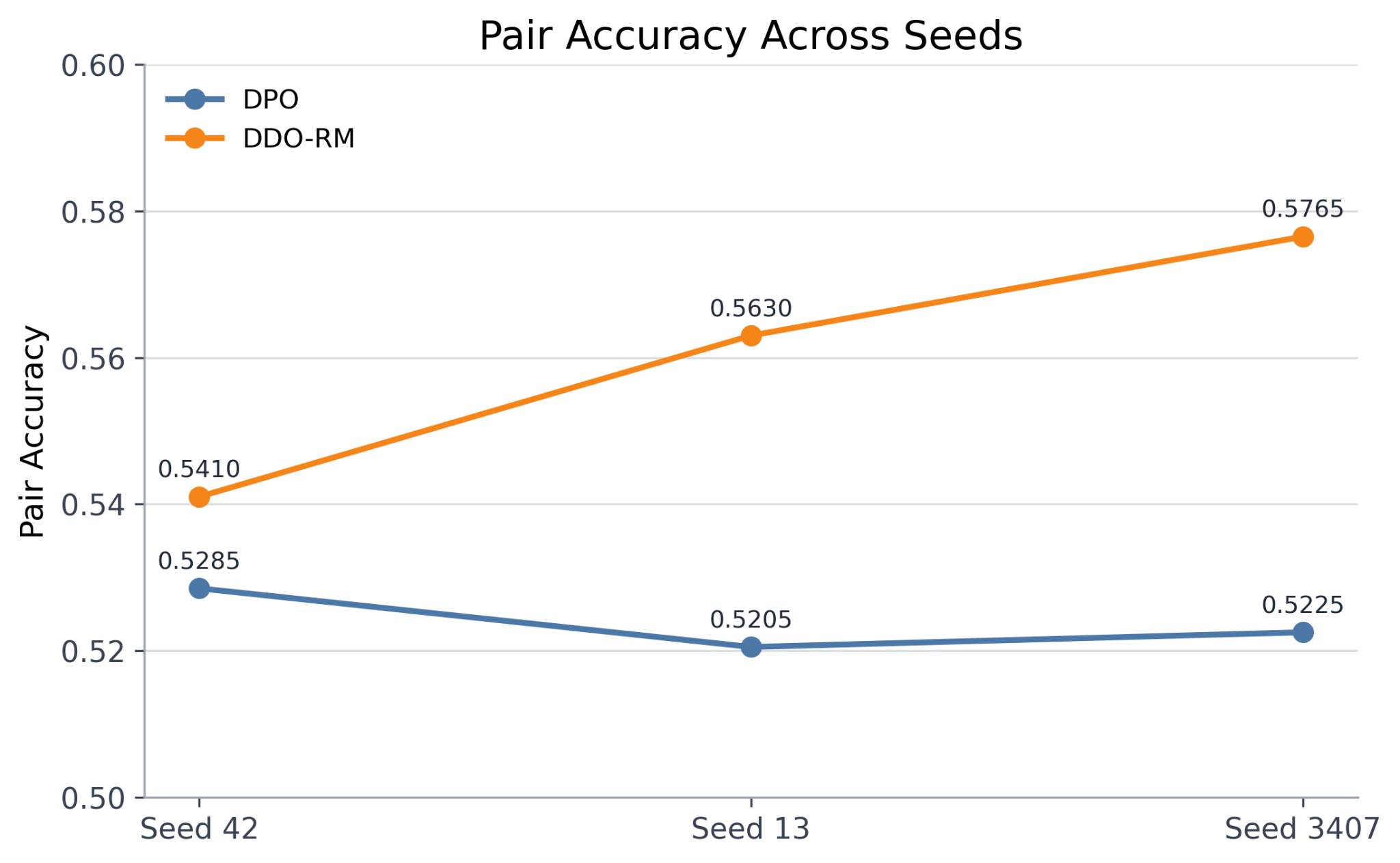}
\caption{Per-seed pair accuracy for DPO and DDO-RM across seeds 42, 13, and 3407. DDO-RM consistently exceeds DPO on every seed, with the gap widening on seeds 13 and 3407.}
\label{fig:seed_lines}
\end{figure}
Table~\ref{tab:results} reports the seed-wise results and the mean across seeds.

\begin{table}[t]
\centering
\caption{Held-out benchmark results on \texttt{test\_prefs}. Higher is better for all metrics.}
\label{tab:results}
\begin{tabular}{llcccc}
\toprule
Metric & Method & Seed 42 & Seed 13 & Seed 3407 & Mean \\
\midrule
\multirow{2}{*}{Pair Accuracy} 
& DPO & 0.5285 & 0.5205 & 0.5225 & 0.5238 \\
& DDO-RM & 0.5410 & 0.5630 & 0.5765 & 0.5602 \\
\midrule
\multirow{2}{*}{AUC}
& DPO & 0.5335 & 0.5301 & 0.5308 & 0.5315 \\
& DDO-RM & 0.5335 & 0.5388 & 0.5423 & 0.5382 \\
\midrule
\multirow{2}{*}{Mean Margin}
& DPO & 0.1308 & 0.1384 & 0.1439 & 0.1377 \\
& DDO-RM & 0.2995 & 0.5196 & 0.7867 & 0.5353 \\
\bottomrule
\end{tabular}
\end{table}

The headline comparison is straightforward.
\begin{itemize}[leftmargin=*]
    \item \textbf{Pair accuracy:} DDO-RM improves the mean from 0.5238 to 0.5602, an absolute gain of 0.0364.
    \item \textbf{AUC:} DDO-RM improves the mean from 0.5315 to 0.5382, an absolute gain of 0.0067.
    \item \textbf{Mean margin:} DDO-RM improves the mean from 0.1377 to 0.5353, an absolute gain of 0.3976.
\end{itemize}

DDO-RM is better than DPO on all three reported seeds for pair accuracy and mean margin. For AUC, the two methods tie on seed 42, while DDO-RM is better on seeds 13 and 3407 and on the overall mean. The margin gap is especially large, which suggests that the reward-guided target not only flips more pairs correctly but also tends to separate preferred and rejected responses more strongly.

\begin{figure}[t]
\centering
\fbox{\begin{minipage}{0.90\linewidth}
\textbf{Summary visualization.} Table~\ref{tab:results} is the authoritative numerical summary. Across three seeds, DDO-RM improves mean pair accuracy, AUC, and mean margin over DPO in this minimal held-out benchmark. The largest absolute gain is on mean margin.
\end{minipage}}
\caption{Compact textual visualization of the current preliminary result. The full numeric values are reported in Table~\ref{tab:results}.}
\label{fig:result_summary_box}
\end{figure}

\takeaway{Experimental takeaway}{The current benchmark is a minimal sanity check, not a final superiority claim. It shows that DDO-RM is viable and improves over DPO in this small held-out setting; the main fair next comparison is PPO-RLHF using the same reward model.}

\paragraph{Interpretation.}
The current results support a narrow statement: within one held-out pairwise preference setting, a reward-model-based distribution update is competitive with a direct pairwise DPO baseline. They do not yet establish general superiority over DPO, PPO-RLHF, or GRPO. In particular, the current study does not include larger models, multiple datasets, confidence intervals, listwise evaluation, or a same-reward-model PPO baseline.

\section{Discussion and Limitations}
\label{sec:discussion}

\subsection{What DDO-RM adds beyond pairwise fitting}

DPO is intentionally pair-centric: it only needs to know which of two responses should rank higher. DDO-RM uses a richer intermediate object, namely a full distribution over candidates. Even when there are only two candidates, this perspective changes the update. Instead of a single binary preference push, DDO-RM forms a reward-guided target that depends on both the current policy and the candidate-level reward landscape. This makes the method especially compatible with reranking, shortlist selection, listwise preference learning, and candidate-based reasoning.

This finite-candidate KL geometry is also the probability-space counterpart of the Plackett--Luce / entropy geometry used in ranking-style mirror descent. In that sense, DDO-RM and RIPLM belong to the same geometric family while serving different roles: RIPLM studies score-space mirror descent for ranking decisions, whereas DDO-RM uses reward-model scores for policy improvement in LLM alignment.

\subsection{Reward-model calibration and robustness}

The same feature that makes DDO-RM appealing also creates a risk. DDO-RM uses reward values to shape a full distribution, so it is more sensitive to reward-model calibration than a method that only uses pairwise order information. Centering the reward under the current policy and applying a KL-regularized projection mitigate this issue, but do not eliminate it. A complete study should test robustness to reward-model noise, scaling, and miscalibration.

This also limits how the RLHF--DPO dichotomy should be interpreted in our paper. That theory identifies regimes where reward-model-first learning can be statistically favorable, but also regimes where reward mis-specification can hurt. DDO-RM inherits both sides of this tradeoff: it may benefit when the reward model is locally consistent over the candidate set, but it can amplify errors if reward scores are poorly calibrated across candidates. This is why our main empirical roadmap emphasizes same-reward-model comparisons and reward-noise stress tests.

\subsection{Fair comparison roadmap}

The current DPO comparison is useful but incomplete. The next experiments should separate three questions.
\begin{enumerate}[leftmargin=*]
    \item \textbf{DDO-Direct versus DPO-family methods.} If no separate reward model is used, a direct distribution-level method should be compared against DPO, ORPO, KTO, and SimPO.
    \item \textbf{DDO-RM versus PPO-RLHF.} If a reward model is used, the most direct comparator is PPO-based RLHF with the same reward model, same prompts, and matched rollout budget.
    \item \textbf{DDO-RM versus GRPO.} In multi-completion or reasoning settings, DDO-RM should be compared against GRPO and related group-relative reward methods.
\end{enumerate}
The original UltraFeedback setting with multiple candidate completions per prompt is especially important, because DDO-RM is most native when the candidate set has size $K>2$.

\takeaway{Limitation takeaway}{The present evidence is deliberately scoped. The method's strongest empirical test is not another binary DPO table, but a same-reward-model PPO-RLHF comparison and a multi-candidate/listwise evaluation.}

\section{Conclusion}

This revised manuscript repositions DDO-RM as distribution-level policy improvement after reward learning. The key conceptual point is that recent theory on the RLHF--DPO dichotomy gives a reason to revisit the reward-model-first route: reward learning can be statistically easier than direct policy fitting in some finite-sample regimes, although that conclusion depends on representation and misspecification assumptions. DDO-RM then asks how to use a learned reward model once it is available.

The proposed answer is an explicit finite-candidate mirror-descent update. Given policy scores and reward-model scores over a candidate set, DDO-RM constructs a KL-regularized reward-improved target distribution and distills it back into the policy. The current Pythia-410M / UltraFeedback Binarized results against DPO are a useful preliminary signal, but they should be interpreted as a minimal reference point rather than a final empirical claim.

The most important next step is a same-reward-model comparison against PPO-RLHF, followed by a multi-candidate comparison against GRPO and listwise baselines. That is the appropriate empirical setting for testing whether explicit distribution-level policy improvement can outperform stochastic or pairwise approximations.


\begin{thebibliography}{99}

\bibitem[Biderman et~al.(2023)]{pythia}
Stella Biderman, Sid Black, Laria Reynolds, et~al.
\newblock Pythia: A suite for analyzing large language models across training and scaling.
\newblock \emph{arXiv preprint arXiv:2304.01373}, 2023.

\bibitem[HuggingFaceH4(2024)]{ufbinarized}
HuggingFaceH4.
\newblock \texttt{ultrafeedback\_binarized} dataset card.
\newblock \url{https://huggingface.co/datasets/HuggingFaceH4/ultrafeedback_binarized}, accessed 2026.

\bibitem[Rafailov et~al.(2023)]{dpo}
Rafael Rafailov, Archit Sharma, Eric Mitchell, Stefano Ermon, Christopher D. Manning, and Chelsea Finn.
\newblock Direct preference optimization: Your language model is secretly a reward model.
\newblock \emph{Advances in Neural Information Processing Systems}, 2023.

\bibitem[Schulman et~al.(2017)]{ppo}
John Schulman, Filip Wolski, Prafulla Dhariwal, Alec Radford, and Oleg Klimov.
\newblock Proximal policy optimization algorithms.
\newblock \emph{arXiv preprint arXiv:1707.06347}, 2017.

\bibitem[Shao et~al.(2024)]{grpo}
Zhihong Shao, Peiyi Wang, Qihao Zhu, Runxin Xu, Junxiao Song, Mingchuan Zhang, Y.~K. Li, Y.~Wu, and Daya Guo.
\newblock DeepSeekMath: Pushing the limits of mathematical reasoning in open language models.
\newblock \emph{arXiv preprint arXiv:2402.03300}, 2024.


\bibitem[Beck and Teboulle(2003)]{beck2003mirror}
Amir Beck and Marc Teboulle.
\newblock Mirror-descent and nonlinear projected subgradient methods for convex optimization.
\newblock \emph{Operations Research Letters}, 31(3):167--175, 2003.

\bibitem[Ethayarajh et~al.(2024)]{kto}
Kawin Ethayarajh, Winnie Xu, Niklas Muennighoff, Dan Jurafsky, and Douwe Kiela.
\newblock KTO: Model alignment as prospect theoretic optimization.
\newblock \emph{arXiv preprint arXiv:2402.01306}, 2024.

\bibitem[Hong et~al.(2024)]{orpo}
Jiwoo Hong, Noah Lee, and James Thorne.
\newblock ORPO: Monolithic preference optimization without reference model.
\newblock \emph{arXiv preprint arXiv:2403.07691}, 2024.

\bibitem[Kivinen and Warmuth(1997)]{kivinen1997exponentiated}
Jyrki Kivinen and Manfred K. Warmuth.
\newblock Exponentiated gradient versus gradient descent for linear predictors.
\newblock \emph{Information and Computation}, 132(1):1--63, 1997.

\bibitem[Meng et~al.(2024)]{simpo}
Yu Meng, Mengzhou Xia, and Danqi Chen.
\newblock SimPO: Simple preference optimization with a reference-free reward.
\newblock \emph{arXiv preprint arXiv:2405.14734}, 2024.

\bibitem[Nemirovsky and Yudin(1983)]{nemirovsky1983problem}
Arkadi Nemirovsky and David Yudin.
\newblock \emph{Problem Complexity and Method Efficiency in Optimization}.
\newblock Wiley, 1983.

\bibitem[Ouyang et~al.(2022)]{ouyang2022training}
Long Ouyang, Jeffrey Wu, Xu Jiang, et~al.
\newblock Training language models to follow instructions with human feedback.
\newblock \emph{Advances in Neural Information Processing Systems}, 2022.

\bibitem[Shi et~al.(2025)]{rlhfdpo}
Ruizhe Shi, Minhak Song, Runlong Zhou, Zihan Zhang, Maryam Fazel, and Simon S. Du.
\newblock Understanding the performance gap in preference learning: A dichotomy of RLHF and DPO.
\newblock \emph{arXiv preprint arXiv:2505.19770}, 2025.


\bibitem[Zhang et~al.(2026)]{ddormrepo}
Tiantian Zhang, Jierui Zuo, and Siyu Lin.
\newblock DDO-RM LLM preference benchmark.
\newblock GitHub repository. \url{https://github.com/zuojr/ddorm-llm-preference-benchmark}, accessed 2026.

\end{thebibliography}
\end{document}